\documentclass[sigconf]{acmart}
\AtBeginDocument{%
  }

\setcopyright{acmlicensed}
\copyrightyear{2026}
\acmYear{2026}
\acmDOI{XXXXXXX.XXXXXXX}
\acmConference[MM '26]{Proceedings of the 34th ACM International Conference on Multimedia}{November 10--14, 2026}{Rio de Janeiro, Brazil}
\acmISBN{978-1-4503-XXXX-X/2026/10}

\renewcommand\footnotetextcopyrightpermission[1]{}

\usepackage[T1]{fontenc}
\usepackage{booktabs}
\usepackage{multirow}
\usepackage{colortbl}
\usepackage{xcolor}
\usepackage{pifont}
\usepackage{amsthm}
\usepackage{makecell}

\definecolor{bestbg}{HTML}{DCEEFB}    
\definecolor{secondbg}{HTML}{FDE8D0}  
\definecolor{midgray}{HTML}{888888}   
\newcommand{\best}{\cellcolor{bestbg}}
\newcommand{\secondbest}{\cellcolor{secondbg}}
\newcommand{\boldres}[1]{\cellcolor{bestbg}\textbf{#1}}
\newcommand{\secondres}[1]{\cellcolor{secondbg}\underline{#1}}

\usepackage{graphicx}

\usepackage{amsmath,amsfonts,bm}

\def\1{\bm{1}}

\def\rvn{{\mathbf{n}}}

\def\rvp{{\mathbf{p}}}

\def\vq{{\bm{q}}}

\def\vt{{\bm{t}}}

\def\vv{{\bm{v}}}

\def\mX{{\bm{X}}}

\DeclareMathAlphabet{\mathsfit}{\encodingdefault}{\sfdefault}{m}{sl}
\SetMathAlphabet{\mathsfit}{bold}{\encodingdefault}{\sfdefault}{bx}{n}

\def\gD{{\mathcal{D}}}

\def\gF{{\mathcal{F}}}

\def\gH{{\mathcal{H}}}

\def\gL{{\mathcal{L}}}
\def\gM{{\mathcal{M}}}

\def\gO{{\mathcal{O}}}

\def\gT{{\mathcal{T}}}

\def\gV{{\mathcal{V}}}

\providecommand{\E}{\mathbb{E}}

\DeclareMathOperator*{\argmin}{arg\,min}

\newtheorem{theorem}{Theorem}
\newtheorem{proposition}[theorem]{Proposition}
\newtheorem{assumption}{Assumption}

\begin{document}

\title{FM$^2$: Unified Federated Foundation Models for Heterogeneous Multimodal Medical Imaging}

\author{Shengchao Chen}
\authornote{Work done during a project-based visiting appointment at Shenzhen University. Email: \texttt{shengchao.chen.uts@gmail.com}.}
\affiliation{%
  \institution{School of Artificial Intelligence, Shenzhen University}
  \city{Shenzhen}
  \country{China}}
\affiliation{%
  \institution{Australian AI Institute, University of Technology Sydney}
  \city{Sydney}
  \country{Australia}}

\author{Ting Shu}
\authornote{Corresponding author. Email: \texttt{tingshu@szu.edu.cn}.}
\affiliation{%
  \institution{School of Artificial Intelligence, Shenzhen University}
  \city{Shenzhen}
  \country{China}}

\renewcommand{\shortauthors}{Shengchao Chen and Ting Shu}

\begin{abstract}
Building foundation models for medical imaging requires pooling data across institutions, yet privacy regulations prohibit centralized aggregation. Existing Federated Foundation Models either fine-tune natural-image models with poor medical-domain transfer, or train from scratch within a single modality, lacking the flexibility to unify tasks. We identify an under-explored challenge, \emph{Imaging Modality Heterogeneity}, where clients operate under two structural regimes: \emph{Overlapped} (shared modalities with heterogeneous label distributions) and \emph{Non-overlapped} (fully disjoint modalities per client). We propose \textbf{FM$^2$}, a unified framework that trains the core backbone from scratch to preserve medical domain fidelity while optionally incorporating biomedical pretrained encoders for vision-language alignment. FM$^2$ equips each client with dual Mixture-of-Experts modules (a Class-wise MoE for personalized category knowledge and a Domain-wise MoE for shared cross-modality representations), coupled with a \emph{Heterogeneous Modality Alignment} (HMA) regularizer that explicitly aligns modality-specific expert parameters, admitting provable $\gO(1/\sqrt{T})$ convergence and generalization guarantees. FM$^2$ further incorporates \emph{Caption-Enhanced Learning (CEL)}, where locally retained GPT-4o-generated captions serve as a textual semantic bridge enabling representation transfer across clients with disjoint modalities, and demonstrates extensibility to \emph{Federated Medical VQA}. Experiments on our MIMH benchmark (classification and CEL) and real-world medical VQA datasets confirm consistent superiority over state-of-the-art federated baselines and strong out-of-modality generalization across all three tasks.
\end{abstract}

\ccsdesc[500]{Computing methodologies~Computer vision}
\ccsdesc[500]{Computing methodologies~Machine learning}

\keywords{Federated Learning, Medical Imaging, Data Heterogeneity, Foundation Models, Mixture-of-Experts, Medical VQA}

\maketitle
\enlargethispage{-17pt}

\section{Introduction}

Foundation Models (FMs) have demonstrated remarkable cross-domain generalization by training large-scale architectures that integrate multiple domain-specific sub-models~\cite{zhang2024challenges}. In medical imaging, building such FMs requires pooling numerous datasets from diverse institutions. However, strict privacy regulations such as GDPR~\cite{voigt2017eu} and HIPAA~\cite{cohen2018hipaa} prohibit the centralized aggregation of sensitive clinical data, creating a fundamental tension between model capability and data governance. Federated Foundation Models (FFMs)~\cite{zhuang2023foundation} resolve this tension by enabling collaborative training across institutions without exposing raw data, making them a natural paradigm for privacy-preserving medical AI.

Recent FFMs for medical imaging predominantly \emph{fine-tune} pretrained natural-image models, including ViT~\cite{dosovitskiy2020image}, CLIP~\cite{radford2021learning}, and SAM~\cite{kirillov2023segment}, on medical downstream tasks. Liu et al.~\cite{liu2024fedfms} fine-tuned SAM for federated medical image segmentation; Wu et al.~\cite{wu2024facmic} adapted CLIP for federated classification. Despite their convenience, these methods inherit a critical limitation: empirical evidence~\cite{huix2024natural} shows that natural-image representations transfer poorly to the medical domain, where subtle inter-class variations, acquisition noise, and precise anatomical context demand domain-native knowledge that pretrained features cannot provide.

An alternative line of work trains FFMs \emph{from scratch} on curated medical data~\cite{liu2024foundation,tolle2024federated}, acquiring richer domain-specific representations, echoing evidence from other modalities that federating the from-scratch training of a foundation model can match centralized counterparts while neutralizing per-domain bias~\cite{chen2026fedal}. Yet these efforts target \emph{single} imaging modalities and \emph{single} tasks. Real-world clinical AI demands a unified framework that simultaneously handles classification, caption-supervised representation learning, and visual question answering across heterogeneous modalities, under both \emph{Overlapped} configurations (clients share modalities but with non-IID label distributions) and \emph{Non-overlapped} configurations (each client holds a completely disjoint modality). Existing federated frameworks address at most a subset of these requirements. The compounding of classical \emph{data heterogeneity} (non-IID labels) with \textbf{Imaging Modality Heterogeneity}, where feature distributions, noise characteristics, and semantic granularity differ fundamentally across clients, makes standard personalized FL methods, designed for label skew alone, insufficient for this setting.

This paper proposes \textbf{FM$^2$} (\textbf{F}ederated \textbf{M}edical \textbf{M}ultimodal foundation model), a framework that trains from scratch under Imaging Modality Heterogeneity and extends them to multimodal medical understanding. The core of FM$^2$ is a dual Mixture-of-Experts (MoE) architecture: a \emph{Class-wise MoE} retains personalized per-category knowledge locally, while a \emph{Domain-wise MoE} captures shared cross-modality representations that are globally aggregated. A \emph{Heterogeneous Modality Alignment} regularizer constrains the domain-wise experts toward global consensus, for which we prove a convergence rate of $\gO(1/\sqrt{T})$ and an explicit generalization bound. Beyond classification, FM$^2$ supports two multimodal extensions: (i)~\emph{Caption-Enhanced Learning (CEL)}, where GPT-4o-generated captions (retained locally for privacy) provide federated vision-language contrastive supervision, serving as a shared semantic bridge that enriches representations even across clients with disjoint visual modalities; and (ii)~\emph{Federated Medical VQA}, demonstrating FM$^2$'s extensibility to natural-language querying of medical images across institutions. Our contributions are summarized as follows:
\begin{itemize}
    \item We formalize \emph{Imaging Modality Heterogeneity} under both Overlapped and Non-overlapped modality regimes, and propose FM$^2$, the first unified multi-task federated framework addressing both settings simultaneously.
    \item We design dual MoE modules with a \emph{Heterogeneous Modality Alignment} regularizer targeting modality-specific expert drift, and prove $\gO(1/\sqrt{T})$ convergence and an explicit generalization bound (Theorem~\ref{thm:convergence}--Proposition~\ref{prop:generalization}).
    \item We introduce \emph{Caption-Enhanced Learning (CEL)}, where locally retained GPT-4o captions provide a shared textual bridge: language descriptions share a common vocabulary across modalities, enabling representation transfer even when clients hold entirely disjoint visual domains.
    \item We construct the \textbf{MIMH} benchmark (five modalities, four configurations) for federated \emph{Classification} and \emph{Caption Enhanced Learning}, and additionally evaluate on real-world medical VQA datasets for \emph{Federated Medical VQA}, confirming consistent superiority across all three tasks with strong out-of-modality generalization.
\end{itemize}

\section{Related Work}
\paragraph{\bf Heterogeneous Federated Learning.} FL trains models collaboratively across decentralized clients without exposing raw data~\cite{mcmahan2017communication}. Statistical heterogeneity (non-IID distributions) is a central challenge: naive aggregation causes client drift and degraded convergence. Personalized FL (PFL) addresses this through several families of techniques. Regularization-based methods~\cite{hanzely2020lower,li2021ditto} add proximal penalties to keep local models close to the global reference while permitting per-client drift; Ditto~\cite{li2021ditto} exemplifies this by solving a personalized objective with a global proximal term, and FFTS~\cite{chen2025federated} extends the idea to foundation models by regularizing both client and server side to align knowledge across statistically heterogeneous datasets. Partial-sharing methods split the network into a globally aggregated backbone and a local task-specific head: FedBN~\cite{li2021fedbn} keeps batch normalization layers local to absorb feature shift, FedRep~\cite{collins2021exploiting} retains only the classifier locally, and prompt- or adapter-based approaches~\cite{chen2023prompt,chen2024personalized} share only lightweight parameters while keeping the backbone on-device. Adaptive aggregation~\cite{zhang2020personalized} weights client contributions by similarity, and meta-learning methods such as Per-FedAvg~\cite{fallah2020personalized} learn a shared initialization that enables rapid per-client fine-tuning. Despite their diversity, all of these approaches assume a homogeneous feature space and optimize for a single task, making them unsuitable when clients differ fundamentally in imaging modality. Mixture-of-experts approaches have also been explored in FL~\cite{marfoq2021federated,xie2025dflmoe}, yet none disentangle class-level from modality-level heterogeneity simultaneously; centralized missing-modality strategies~\cite{xu2023multimodal} are further inapplicable since FL prohibits raw cross-site feature aggregation.

\paragraph{\bf FL for Medical Images.} Federated learning has been applied to privacy-preserving medical image classification~\cite{ren2024federated,chen2023interpretable,feng2026visual,wang2024ensemble} and, more recently, to vision-language tasks including medical VQA~\cite{liu2021slake,lau2018dataset}. For classification, methods such as FedFMS~\cite{liu2024fedfms} and FedTCA~\cite{chen2025restyled} are typically task-specific and backbone-dependent, requiring re-engineering for each setting and assuming a single shared imaging modality. For vision-language understanding, FedDAT~\cite{chen2024feddat}, PromptFL~\cite{guo2023promptfl}, and F$^3$OCUS~\cite{saha2025f} adapt pretrained VLMs to federated settings via dual adapters, prompt tuning, and client-layer meta-learning, but remain confined to single-modality client scenarios. No existing method jointly supports classification and VQA under a protocol that handles cross-modality structure. FM$^2$ fills this gap by disentangling class-level personalization from modality-level consensus via dual MoE modules, enabling unified multi-task training across heterogeneous imaging modalities without task-specific re-engineering.
\section{Methodology}

\subsection{Overview}

\textbf{Fig.~\ref{fig:overall_framework}} illustrates the FM$^2$ framework. Each client first trains a local visual encoder $\gF(\cdot) = \{\gF_b, \gF_h\}$ on its own data. The backbone $\gF_b$ then serves as a frozen \emph{Observer} $\gV$ that performs Top-$K$ Sampling to extract representative subsets $\gD_1, \ldots, \gD_N$ per class, forming a sampled dataset $\gD_s = \bigcup_{i=1}^N \gD_i$. Two MoE modules, namely Class-wise MoE $\gM_c$ and Domain-wise MoE $\gM_d$, are trained on $\gD_s$ following $\gV$. After local updates, $\gV$ and $\gM_d$ are uploaded for global aggregation; $\gM_c$ remains local for personalization. $\gV$ is frozen during stage-2 MoE training but uploaded at round end so the server can aggregate and redistribute a refined backbone for the next round's stage-1 initialization. For multimodal extension, an optional text encoder $\gT$ processes expert-generated captions, and a VQA head $\gH_{\text{vqa}}$ enables question answering, both sharing the same MoE-enhanced visual backbone.

\begin{figure*}[t]
    \centering
    \includegraphics[width=.92\textwidth]{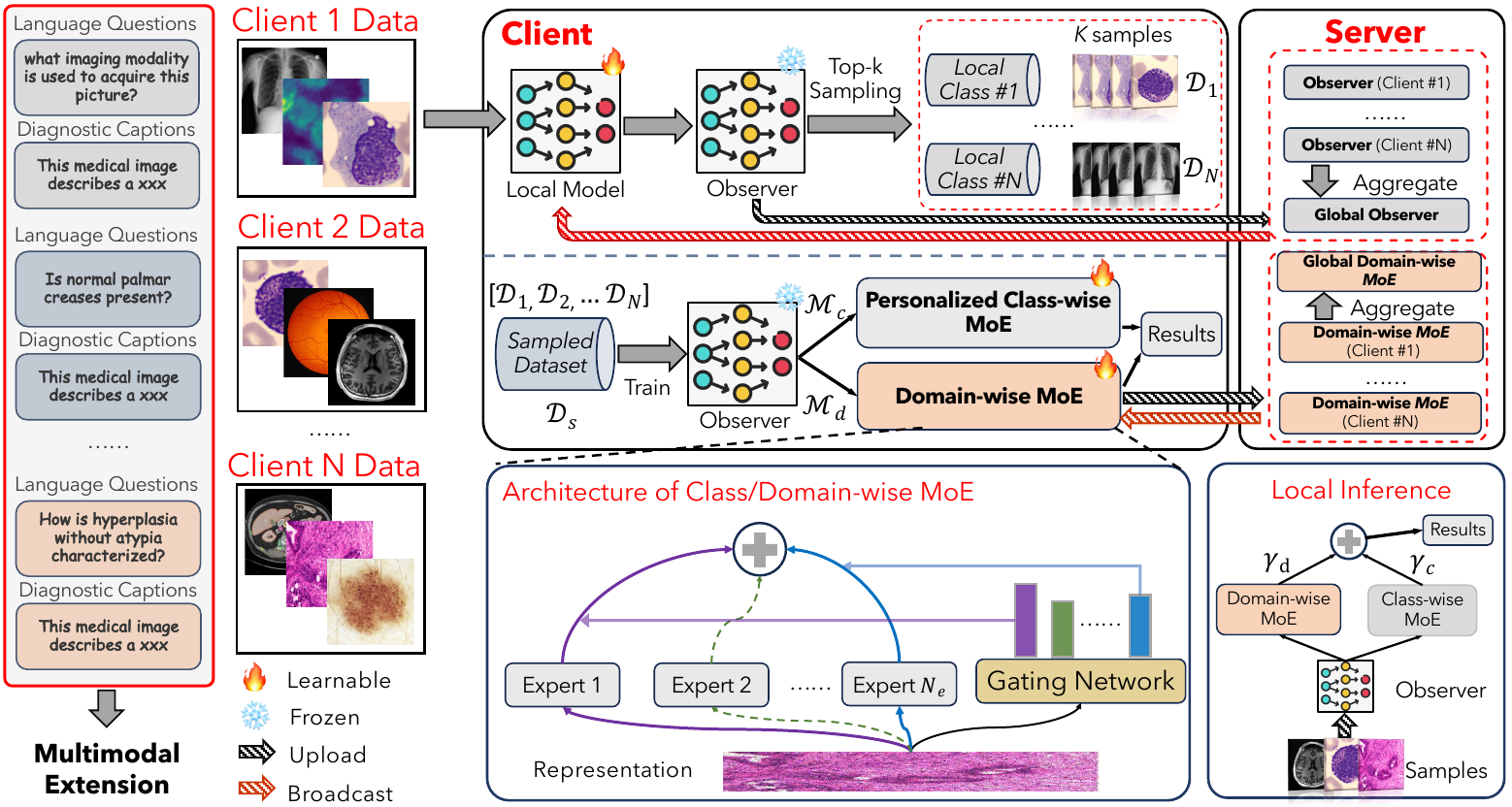}
    \caption{FM$^2$ (single-client view). Observer extracts representative samples for training dual MoE modules. Domain-wise MoE is globally aggregated; Class-wise MoE remains local. Text inputs (diagnostic captions for CEL; questions for VQA) are processed by a local text encoder and fused with MoE-enhanced visual features, enabling all downstream tasks within a single protocol.}
    \label{fig:overall_framework}
\end{figure*}

\subsection{Dual Mixture-of-Experts for Heterogeneity}

\paragraph{Observer and Top-$K$ Sampling.}
In the first training stage, each client trains its local model $\gF$ with labeled images. The trained backbone, designated as the Observer $\gV = \gF_b$, computes a representativeness score $S = \texttt{Sigmoid}(\gV(\mX_z))$ for each class $z$. Top-$K$ Sampling selects the $K$ highest-scoring samples per class, forming subsets $\gD_1, \ldots, \gD_N$ (sensitivity analysis of $K$ in Appendix~B). This observer-guided selection addresses local label drift by ensuring that only the most informative samples enter the second stage.

\paragraph{Class-wise MoE (CMoE)}
To handle label heterogeneity, the CMoE $\gM_c$ allocates $N_c$ experts (one per local class, $N_c = N$), each specializing in class-specific feature transformations:
\begin{equation}\label{eq:cmoe}
    \hat{\mX} = \sum_{\rvn=1}^{N_c} W_{\rvn} \cdot E_{\rvn}(\mX_{\text{mid}}),\quad \text{where}\; E(\mX_{\text{mid}}) = \textsc{ReLU}(\textsc{MLP}(\mX_{\text{mid}})),
\end{equation}
where $\mX_{\text{mid}}$ and $\hat{\mX}$ denote latent and output representations, $E_{\rvn}$ is the $\rvn$-th expert network, and $W_{\rvn}$ is the gating weight. By decomposing the representation space into class-specific subspaces, CMoE transforms the classification problem into learning multiple potential class-specific knowledge bases, each tailored to the client's local label distribution. The CMoE remains on each client and is never uploaded, preserving personalization.

\paragraph{Domain-wise MoE (DMoE)}
To handle modality heterogeneity, the DMoE $\gM_d$ allocates $N_d$ experts (one per imaging modality), each capturing modality-specific feature transformations. Its architecture mirrors CMoE but with experts indexed by imaging domain rather than class. Expert indices are pre-assigned via a shared global modality vocabulary declared at federation setup, reflecting the practical assumption that each clinical site knows its imaging modality (e.g., a radiology department knows it operates CT scanners); extension to unknown-modality scenarios via online feature clustering is discussed in Appendix~C. The DMoE is globally aggregated, enabling cross-client knowledge transfer of domain representations.

For inference, the final prediction combines both MoE modules:
\begin{equation}\label{eq:eq2}
    y = \gamma_d \cdot \gM_d(\gV(\mX)) + \gamma_c \cdot \gM_c(\gV(\mX)),
\end{equation}
where $\gamma_d, \gamma_c \in [0,1]$ with $\gamma_d + \gamma_c = 1$ balance domain consensus and personalized expertise.

\subsection{Heterogeneous Modality Alignment}

Aggregating domain-wise experts across clients with fundamentally different imaging modalities risks introducing consensus bias. We mitigate this through a regularization that explicitly constrains the local DMoE parameters $\gM_d^i$ toward the global aggregate $\hat{\gM}_d$:
\begin{equation}\label{eq:optimization}
\begin{aligned}
    &\min_{\{\theta_1, \ldots, \theta_Z\}}  \sum_{i=1}^Z \frac{n_i}{n} \bigl[\gL_i(\theta_i;\gD^i, \gD_s^i) + \lambda \|\gM_d^i - \hat{\gM}_d\|^2\bigr], \\
    &\text{s.t.}\quad \hat{\gM}_d = \argmin_{\gM_d} G(\gM_d^1, \ldots, \gM_d^Z), \\
    &\text{where}\quad \gL_i(\theta_i;\gD^i, \gD_s^i) = -\sum_{c=1}^C y_c \log(\hat{y}_i(x)),\; \theta_i = \{\gV^i, \gM_d^i, \gM_c^i\}.
\end{aligned}
\end{equation}
Here $\lambda > 0$ controls the alignment strength, $n_i / n$ weights clients by data volume, and $\hat{\gM}_d$ is obtained via FedAvg. The proximal term $\lambda\|\gM_d^i - \hat{\gM}_d\|^2$ serves two purposes: it prevents local domain experts from drifting too far from the global consensus while still allowing modality-specific adaptation. Unlike FedProx~\cite{li2020federated} and Ditto~\cite{li2021ditto}, which apply a global proximal penalty to the \emph{entire} model, HMA constrains only the DMoE parameters, exploiting the dual-MoE structure to selectively align modality-specific representations while leaving CMoE fully free for per-client personalization. We establish formal convergence guarantees for this  in Sec.~\ref{sec:theory}.

\subsection{Multimodal Extensions}\label{sec:multimodal}

The dual MoE architecture provides a modality-aware visual backbone that naturally supports richer output modalities beyond classification. We introduce two extensions that operate on the same MoE-enhanced features without modifying the core framework.

\paragraph{Caption-Enhanced Learning (CEL)}
Medical images encode rich diagnostic information that class labels alone cannot capture, such as lesion morphology, tissue texture, and spatial relationships. This limitation is particularly acute in the Non-overlapped setting (MIMH-SA), where clients hold entirely disjoint imaging modalities: visual feature distributions share no common structure across clients, yet their textual descriptions share a common diagnostic vocabulary regardless of modality. CEL exploits this asymmetry by aligning local visual features with LLM-generated captions through a federated contrastive objective, so that language serves as a shared semantic bridge enabling cross-client representation transfer even when visual domains are completely disjoint. To inject this knowledge, we generate diagnostic captions for each image in the MIMH dataset using GPT-4o~\cite{openai2024gpt4o} via a ReAct-based~\cite{yao2023react} agent pipeline (full prompts and quality-filtering criteria in Appendix~D) in a zero-shot setting (e.g., \emph{``Fundus image showing microaneurysms and hard exudates in the macular region, consistent with moderate non-proliferative diabetic retinopathy''}). A lightweight text encoder $\gT$ (PubMedBERT~\cite{gu2021domain}) extracts caption embeddings $\vt = \gT(\text{caption})$, which are aligned with the visual features $\vv = \gV(\mX)$ through the federated contrastive objective:
\begin{equation}\label{eq:cel}
    \gL_{\text{CEL}} = -\frac{1}{B}\sum_{j=1}^{B} \log \frac{\exp(\text{sim}(\vv_j, \vt_j) / \tau)}{\sum_{k=1}^{B} \exp(\text{sim}(\vv_j, \vt_k) / \tau)},
\end{equation}
where $\text{sim}(\cdot, \cdot)$ denotes cosine similarity, $\tau$ is a temperature parameter, and $B$ is the batch size. The total local objective becomes $\gL_i + \alpha \cdot \gL_{\text{CEL}}$, where $\alpha$ balances classification and contrastive losses. Crucially, $\gL_{\text{CEL}}$ enriches the Observer's representations with fine-grained semantic structure \emph{without requiring caption data to leave the client}, preserving privacy. The text encoder parameters remain local; only the visual backbone and DMoE are aggregated.

\paragraph{Federated Medical VQA}
To demonstrate FM$^2$'s capacity for multimodal reasoning, we extend it to Visual Question Answering on medical images. Given an image $\mX$ and a natural-language question $q$, we extract visual features $\vv = \gM_d(\gV(\mX)) \oplus \gM_c(\gV(\mX))$ (where $\oplus$ denotes concatenation weighted by $\gamma_d, \gamma_c$) and question features $\vq = \gT_q(q)$ via a question encoder $\gT_q$. A lightweight answer module $\gH_{\text{vqa}}$ fuses both:
\begin{equation}\label{eq:vqa}
    a = \gH_{\text{vqa}}(\vv \odot \vq),
\end{equation}
where $\odot$ denotes element-wise gating. The VQA loss $\gL_{\text{VQA}}$ (binary cross-entropy over candidate answers) is added to the local objective. The DMoE's domain-specific experts naturally handle modality-dependent questions (e.g., questions about retinal layers vs.\ tissue staining), while the CMoE provides category-aware visual reasoning. The question encoder and answer head are locally trained; the shared visual backbone benefits from federated aggregation across all tasks.

\subsection{Theoretical Analysis}\label{sec:theory}

We provide convergence and generalization guarantees for the FM$^2$ optimization defined in Eq.~\eqref{eq:optimization}. For clarity, we write the global objective as $F(\theta) = \sum_{i=1}^Z \frac{n_i}{n}\bigl[\gL_i(\theta_i) + \lambda\|\gM_d^i - \hat{\gM}_d\|^2\bigr]$, where each client's alignment penalty is included inside the per-client term.

\begin{assumption}\label{asm:smooth}
Each local loss $\gL_i(\theta)$ is $L$-smooth: $\|\nabla \gL_i(\theta) - \nabla \gL_i(\theta')\| \leq L\|\theta - \theta'\|$ for all $\theta, \theta'$. As the alignment penalty is quadratic with curvature $2\lambda$, $F$ is $(L{+}2\lambda)$-smooth; we absorb this into $L$ for brevity.
\end{assumption}

\begin{assumption}\label{asm:variance}
The stochastic gradient $g_i$ at client $i$ is unbiased, $\E[g_i] = \nabla \gL_i(\theta)$, satisfies $\E\|g_i - \nabla \gL_i(\theta)\|^2 \leq \sigma^2$, and its noise is independent across clients.
\end{assumption}

\begin{assumption}\label{asm:heterogeneity}
The modality heterogeneity is bounded:
\[\textstyle\frac{1}{Z}\sum\nolimits_{i=1}^Z \|\nabla \gL_i(\theta) - \nabla F(\theta)\|^2 \leq \delta^2.\]
\end{assumption}

\begin{theorem}[Convergence of FM$^2$]\label{thm:convergence}
Under Assumptions~\ref{asm:smooth}--\ref{asm:heterogeneity}, with learning rate $\eta \leq \min\{1/(4L),\,1/(2\lambda)\}$ and alignment parameter $\lambda > 0$, after $T$ communication rounds of FM$^2$, the averaged squared gradient norm satisfies:
\begin{equation}\label{eq:convergence}
\begin{split}
    \frac{1}{T}\sum_{t=0}^{T-1} \E\|\nabla F(\theta^t)\|^2 \leq\;
    & \underbrace{\frac{2(F(\theta^0) - F^*)}{\eta T}}_{\text{init.\ gap}}
    + \underbrace{\frac{4L\eta\sigma^2}{Z}}_{\text{stoch.\ noise}} \\
    & + \underbrace{4L\eta\delta^2}_{\text{heterogeneity}}
    + \underbrace{\frac{2\lambda L \Delta_{\gM}}{T}}_{\text{align.\ drift}},
\end{split}
\end{equation}
where $F^* = \inf_\theta F(\theta)$ and $\Delta_{\gM} = \max_i \|\gM_d^{i,0} - \hat{\gM}_d^0\|^2$ is the initial modality divergence.
\end{theorem}

\begin{proof}
By the $L$-smoothness of $\gL_i$ (Assumption~\ref{asm:smooth}), the global objective $F$ satisfies:
\begin{equation}\label{eq:descent}
    F(\theta^{t+1}) \leq F(\theta^t) + \langle \nabla F(\theta^t), \theta^{t+1} - \theta^t \rangle + \frac{L}{2}\|\theta^{t+1} - \theta^t\|^2.
\end{equation}
Substituting the update $\theta^{t+1} = \theta^t - \eta \bar{g}^t$, where $\bar{g}^t = \sum_{i=1}^Z p_i\bigl[g_i^t + 2\lambda(\gM_d^{i,t} - \hat{\gM}_d^t)\bigr]$ with $p_i = n_i/n$ is the aggregated gradient including the alignment term, we obtain:
\begin{equation}
    F(\theta^{t+1}) \leq F(\theta^t) - \eta \langle \nabla F(\theta^t), \bar{g}^t \rangle + \frac{L\eta^2}{2}\|\bar{g}^t\|^2.
\end{equation}
Taking expectation and decomposing $\E\|\bar{g}^t\|^2$, we bound the variance term using Assumption~\ref{asm:variance}: $\E\|\bar{g}^t - \nabla F(\theta^t)\|^2 \leq \sigma^2/Z + \delta^2 + 4\lambda^2 \Delta_{\gM}^t$, where $\Delta_{\gM}^t = \frac{1}{Z}\sum_i \|\gM_d^{i,t} - \hat{\gM}_d^t\|^2$.

To control $\Delta_{\gM}^t$, observe that the gradient of the alignment penalty with respect to $\gM_d^i$ is $2\lambda(\gM_d^{i,t} - \hat{\gM}_d^t)$. After the local gradient step, $\gM_d^{i,t+1} = \gM_d^{i,t} - 2\lambda\eta(\gM_d^{i,t} - \hat{\gM}_d^t) - \eta \nabla_{\gM_d^i}\gL_i^t$. Since the server aggregate satisfies $\hat{\gM}_d^{t+1} = \frac{1}{Z}\sum_i \gM_d^{i,t+1}$, the residual $\gM_d^{i,t+1} - \hat{\gM}_d^{t+1} = (1-2\lambda\eta)(\gM_d^{i,t} - \hat{\gM}_d^t) - \eta(\nabla_{\gM_d^i}\gL_i^t - \overline{\nabla\gL}^t)$. Taking norms, averaging, and applying Young's inequality with $\beta = 2\lambda\eta/(1-2\lambda\eta)$ yields $\Delta_{\gM}^{t+1} \leq (1-2\lambda\eta)\Delta_{\gM}^t + \frac{\eta}{2\lambda}\delta^2$, where $2\lambda\eta \leq 1$ holds because $\eta \leq 1/(2\lambda)$. Telescoping gives $\frac{1}{T}\sum_{t=0}^{T-1}\Delta_{\gM}^t \leq \frac{\Delta_{\gM}}{2\lambda\eta T} + \frac{\delta^2}{4\lambda^2}$; the second term enters the bound as $\gO(L\eta\delta^2)$ and is absorbed into the heterogeneity term, while the first, scaled by the $4\lambda^2$ coefficient and the $L\eta$ prefactor of the variance, contributes $2\lambda L \Delta_{\gM}/T$.

Combining, using $\eta \leq 1/(4L)$, and telescoping over $T$ rounds:
\begin{equation}
    \frac{1}{T}\sum_{t=0}^{T-1} \E\|\nabla F(\theta^t)\|^2 \leq \frac{2(F(\theta^0) - F^*)}{\eta T} + \frac{4L\eta\sigma^2}{Z} + 4L\eta\delta^2 + \frac{2\lambda L \Delta_{\gM}}{T}.
\end{equation}
Setting $\eta = \gO(1/\sqrt{T})$ yields the overall convergence rate $\gO(1/\sqrt{T})$.
\end{proof}

The bound shows that $\sigma^2/Z$ diminishes with more clients; $\delta^2$ captures modality heterogeneity reduced by dual MoE disentanglement; and $2\lambda L\Delta_{\gM}/T$ vanishes at $\gO(1/T)$, confirming HMA progressively eliminates modality drift.

\begin{proposition}[Generalization with Dual MoE]\label{prop:generalization}
Let $h = \gamma_d \cdot \gM_d \circ \gV + \gamma_c \cdot \gM_c \circ \gV$ be the FM$^2$ predictor, with $\gamma_d, \gamma_c \geq 0$ and $\gamma_d + \gamma_c = 1$ as in Eq.~\eqref{eq:eq2}. For any joint distribution $P$ over images and labels with modality indicator $m(\mX)$ and any loss $\ell \in [0,1]$ convex in its first argument, the expected risk admits the decomposition:
\begin{equation}\label{eq:generalization}
    R(h) \leq \gamma_d \cdot R_d^* + \gamma_c \cdot R_c^* + 2\sqrt{2\gamma_d \gamma_c\, d_{\mathrm{JS}}(P_d \| P_c)} + \epsilon_{\gV},
\end{equation}
where $R_d^*$ and $R_c^*$ are the Bayes-optimal risks of the domain-wise and class-wise experts respectively, $d_{\mathrm{JS}}(\cdot\|\cdot)$ denotes the Jensen--Shannon divergence between the domain and class expert output distributions, and $\epsilon_{\gV}$ is the Observer's approximation error.
\end{proposition}
\begin{table*}[tbh]
\centering
\caption{Federated classification accuracy on MIMH-5/4/3-Domain. We vary the client join ratio $r = \{30\%, 50\%, 100\%\}$ to simulate different participation levels. \colorbox{bestbg}{\textbf{Blue}}: best, \colorbox{secondbg}{\underline{Peach}}: second-best.}
\vspace{-6pt}
\resizebox{.9\textwidth}{!}{
\begin{tabular}{ll*{9}{c}}
\toprule
\multirow{2}{*}{\textbf{Dataset}} & \multirow{2}{*}{\textbf{Method}} & \multicolumn{3}{c}{\textbf{$\varphi=0.1$} (\textbf{Hard})} & \multicolumn{3}{c}{$\varphi=0.3$ (\textbf{Moderate})} & \multicolumn{3}{c}{\textbf{$\varphi=100$} (\textbf{Easy})} \\
\cmidrule(lr){3-5} \cmidrule(lr){6-8} \cmidrule(lr){9-11}
& & $r = 30\%$ & $r = 50\%$ & $r = 100\%$ & $r = 30\%$ & $r = 50\%$ & $r = 100\%$ & $r = 30\%$ & $r = 50\%$ & $r = 100\%$ \\
\midrule
\multirow{9}{*}{5-Domain}
                         & FedAvg      & $18.67_{(\pm0.30)}$ & $29.68_{(\pm0.21)}$ & $23.05_{(\pm0.47)}$  & $33.06_{(\pm1.40)}$     & $39.29_{(\pm0.20)}$     & $36.43_{(\pm0.22)}$     & $27.78_{(\pm0.29)}$  & $34.47_{(\pm0.62)}$   & $31.90_{(\pm0.92)}$     \\
                         & FedProx     & $47.94_{(\pm0.23)}$ & $34.99_{(\pm0.60)}$ & $38.70_{(\pm0.12)}$ & $33.61_{(\pm0.21)}$ & $36.26_{(\pm0.05)}$ & $39.53_{(\pm0.34)}$ & $31.25_{(\pm0.34)}$ &  $35.70_{(\pm0.07)}$   & $31.61_{(\pm0.79)}$     \\
                         & PerFedAvg   & $29.42_{(\pm0.50)}$ & $39.75_{(\pm0.51)}$ &  $50.91_{(\pm0.23)}$ & $44.18_{(\pm0.29)}$ & $38.52_{(\pm0.27)}$ & $43.06_{(\pm0.52)}$ &$29.51_{(\pm0.41)}$ & $51.61_{(\pm0.39)}$     & $44.71_{(\pm0.54)}$     \\
                       & FedBN       & $25.95_{(\pm0.41)}$ & $33.79_{(\pm0.81)}$ & $42.72_{(\pm0.41)}$ & $29.93_{(\pm0.47)}$ & $39.33_{(\pm0.04)}$ & $39.59_{(\pm0.06)}$  & $24.79_{(\pm0.35)}$ & $37.91_{(\pm0.28)}$     & $35.01_{(\pm0.76)}$    \\
                       & FedProto    & $42.49_{(\pm0.72)}$ & $57.51_{(\pm0.26)}$ & $70.84_{(\pm1.17)}$ & $39.50_{(\pm0.13)}$ & $45.09_{(\pm0.86)}$ & $56.89_{(\pm0.49)}$ & $52.78_{(\pm0.77)}$ & \secondbest $68.42_{(\pm0.17)}$     & \secondbest $85.00_{(\pm1.71)}$   \\
                       & FedRep      & \secondbest $66.97_{(\pm0.34)}$ & \secondbest $73.47_{(\pm0.29)}$ & \secondbest $74.72_{(\pm0.56)}$ & \secondbest $60.32_{(\pm1.28)}$ & \secondbest $65.63_{(\pm0.32)}$ & \secondbest $60.28_{(\pm0.61)}$ & \secondbest $52.87_{(\pm0.99)}$ & $64.48_{(\pm0.47)}$     & $62.01_{(\pm1.01)}$   \\
                         \cmidrule(lr){2-11}
                         & \textbf{FM$^2$-Tiny} &\best $74.05_{(\pm0.74)}$ &\best $79.21_{(\pm0.50)}$ &\best $83.79_{(\pm0.45)}$ &\best $63.03_{(\pm0.85)}$ &\best $69.76_{(\pm0.26)}$ &\best $63.73_{(\pm0.91)}$ &\best $60.80_{(\pm0.14)}$ &\best $73.36_{(\pm0.39)}$ &\best $87.11_{(\pm0.58)}$     \\
                         & \textbf{FM$^2$-Medium} &\best $74.83_{(\pm 0.53)}$ &\best $82.14_{(\pm 0.85)}$ &\best $85.66_{(\pm 0.29)}$ &\best $66.96_{(\pm 0.70)}$ &\best $68.05_{(\pm 0.61)}$ &\best $69.51_{(\pm 0.70)}$ &\best $62.94_{(\pm 0.12)}$ &\best $74.27_{(\pm 0.80)}$ &\best $89.68_{(\pm 0.93)}$    \\
                         & \textbf{FM$^2$-Large} &\best $76.05_{(\pm 0.84)}$ &\best $83.74_{(\pm 0.92)}$ &\best $86.68_{(\pm 0.75)}$ &\best $67.53_{(\pm 0.93)}$ &\best $69.79_{(\pm 0.62)}$ &\best $70.49_{(\pm 0.99)}$ &\best  $63.29_{(\pm 0.22)}$ &\best  $76.73_{(\pm 0.41)}$ &\best  $90.39_{(\pm 0.64)}$   \\
\midrule
\multirow{9}{*}{4-Domain}
                         & FedAvg      & $29.79_{(\pm0.21)}$ & $34.05_{(\pm0.02)}$ & $43.81_{(\pm0.78)}$  & $42.01_{(\pm0.30)}$     & $49.81_{(\pm0.77)}$    & $51.50_{(\pm1.26)}$    & $27.59_{(\pm0.11)}$  & $25.07_{(\pm0.78)}$     & $30.80_{(\pm1.11)}$      \\
                         & FedProx     & $28.78_{(\pm0.13)}$ & $46.12_{(\pm0.41)}$ & $40.33_{(\pm0.16)}$ & $41.20_{(\pm0.45)}$ & $47.75_{(\pm0.88)}$  & $52.06_{(\pm0.65)}$ & $27.41_{(\pm0.24)}$ & $33.01_{(\pm0.26)}$     & $30.06_{(\pm0.31)}$     \\
                         & PerFedAvg   & $41.19_{(\pm0.12)}$ & $42.59_{(\pm0.30)}$  & $45.71_{(\pm0.10)}$ & $56.56_{(\pm0.12)}$ & $53.84_{(\pm0.79)}$ & $58.64_{(\pm0.22)}$ & $34.00_{(\pm0.31)}$ & $33.31_{(\pm0.17)}$     & $37.78_{(\pm0.88)}$     \\
                         & FedBN       & $36.36_{(\pm0.09)}$ & $47.32_{(\pm0.44)}$ & $48.65_{(\pm0.09)}$ & $40.38_{(\pm0.12)}$ & $53.40_{(\pm0.31)}$ &  $55.77_{(\pm1.23)}$ & $21.22_{(\pm0.31)}$ & $29.55_{(\pm0.59)}$     & $32.95_{(\pm0.47)}$     \\
                         & FedProto    & \secondbest $61.07_{(\pm0.11)}$ & \secondbest $76.11_{(\pm0.12)}$ &\best  $88.14_{(\pm0.46)}$ & $60.01_{(\pm0.05)}$ & $69.21_{(\pm0.77)}$ & $77.38_{(\pm0.26)}$ & \secondbest $64.21_{(\pm0.09)}$ & \secondbest $73.04_{(\pm0.78)}$     & \secondbest $92.71_{(\pm0.56)}$     \\
                         & FedRep      &\best  $67.81_{(\pm0.02)}$ &\best  $79.40_{(\pm1.00)}$ & \secondbest $81.28_{(\pm0.31)}$ & \secondbest $64.37_{(\pm0.14)}$ & \secondbest $75.07_{(\pm0.31)}$ & \secondbest $79.02_{(\pm0.21)}$ &\best  $79.31_{(\pm0.27)}$ &\best  $80.81_{(\pm0.24)}$     & $84.92_{(\pm0.69)}$     \\
                         \cmidrule(lr){2-11}
                         & \textbf{FM$^2$-Tiny} & $66.16_{(\pm0.46)}$ & $76.68_{(\pm0.59)}$ & $87.84_{(\pm0.39)}$ &\best  $69.01_{(\pm0.22)}$ &\best  $79.74_{(\pm0.48)}$ & \best $82.66_{(\pm0.91)}$ & $79.11_{(\pm0.63)}$ & $80.34_{(\pm0.33)}$ &\best  $95.36_{(\pm0.15)}$    \\
                         & \textbf{FM$^2$-Medium} &\best  $69.55_{(\pm 0.48)}$ &\best  $79.56_{(\pm 0.62)}$ &\best  $91.05_{(\pm 0.35)}$ &\best  $70.88_{(\pm 0.45)}$ &\best  $81.99_{(\pm 0.74)}$ &\best  $83.41_{(\pm 0.80)}$ &\best  $82.26_{(\pm 0.16)}$ &\best  $85.37_{(\pm 0.34)}$ &\best  $95.63_{(\pm 0.64)}$     \\
                         & \textbf{FM$^2$-Large} &\best  $70.85_{(\pm 0.42)}$ &\best  $82.56_{(\pm 0.58)}$ &\best  $93.12_{(\pm 0.41)}$ &\best  $72.74_{(\pm 0.56)}$ &\best  $85.83_{(\pm 0.69)}$ &\best  $86.37_{(\pm 0.34)}$ &\best  $83.73_{(\pm 0.41)}$ &\best  $86.55_{(\pm 0.26)}$ &\best  $96.00_{(\pm 0.50)}$ \\
\midrule
\multirow{9}{*}{3-Domain}
                         & FedAvg     & $35.16_{(\pm0.76)}$ & $47.14_{(\pm0.24)}$ & $52.58_{(\pm0.40)}$  & $49.50_{(\pm0.10)}$     & $52.62_{(\pm0.33)}$     & $54.18_{(\pm0.28)}$ & $42.82_{(\pm0.10)}$ & $51.63_{(\pm0.16)}$  & $49.32_{(\pm0.31)}$     \\
                         & FedProx     & $42.40_{(\pm0.11)}$ & $46.96_{(\pm0.21)}$ & $52.70_{(\pm0.06)}$ & $52.39_{(\pm0.30)}$ & $51.50_{(\pm0.55)}$ & $57.91_{(\pm0.18)}$ &  $37.03_{(\pm0.44)}$ & $44.37_{(\pm0.62)}$     & $52.18_{(\pm0.53)}$    \\
                         & PerFedAvg   & $40.43_{(\pm0.13)}$ & $51.56_{(\pm0.25)}$ & $58.40_{(\pm0.26)}$ & $57.81_{(\pm0.62)}$ & $55.72_{(\pm0.40)}$ & $59.15_{(\pm0.08)}$ &  $49.86_{(\pm0.22)}$ & $52.10_{(\pm0.10)}$     & $59.11_{(\pm0.03)}$     \\
                         & FedBN       & $44.15_{(\pm0.04)}$ & $53.72_{(\pm0.22)}$ & $59.65_{(\pm0.36)}$ & $44.13_{(\pm0.26)}$ & $57.64_{(\pm0.36)}$ & $57.98_{(\pm0.38)}$  & $40.85_{(\pm0.14)}$ & $50.25_{(\pm0.77)}$   & $51.59_{(\pm0.36)}$     \\
                         & FedProto    & \secondbest $61.43_{(\pm0.20)}$ & \secondbest $80.29_{(\pm0.21)}$ & \secondbest $81.55_{(\pm0.98)}$ & \secondbest $68.24_{(\pm0.55)}$ & $72.43_{(\pm0.16)}$ & $80.78_{(\pm1.08)}$ & $68.43_{(\pm1.43)}$ & \secondbest $80.20_{(\pm1.06)}$     & \best $82.01_{(\pm0.99)}$     \\
                         & FedRep      & \best $77.89_{(\pm0.30)}$ & $76.91_{(\pm0.45)}$ & $78.96_{(\pm0.76)}$ & $63.73_{(\pm0.18)}$ & \secondbest $77.83_{(\pm0.53)}$  & \secondbest $83.14_{(\pm0.24)}$ & \secondbest $71.24_{(\pm0.20)}$ & $70.57_{(\pm0.11)}$     & \secondbest $75.62_{(\pm0.32)}$     \\
                         \cmidrule(lr){2-11}
                         & \textbf{FM$^2$-Tiny} &  $76.93_{(\pm0.35)}$ & \best $83.03_{(\pm0.57)}$ & \best $89.43_{(\pm0.54)}$ & \best $77.83_{(\pm0.40)}$ & \best $83.41_{(\pm0.25)}$ & \best $85.12_{(\pm0.88)}$ & \best $73.24_{(\pm0.62)}$ & \best $83.23_{(\pm0.51)}$ & $81.65_{(\pm0.35)}$  \\
                         & \textbf{FM$^2$-Medium} & \best $78.72_{(\pm 0.45)}$ & \best $86.03_{(\pm 0.38)}$ & \best $90.03_{(\pm 0.17)}$ & \best $78.87_{(\pm 0.34)}$ & \best$83.98_{(\pm 0.72)}$ & \best $87.99_{(\pm 0.99)}$ & \best $74.53_{(\pm 0.60)}$ & \best $84.58_{(\pm 0.29)}$ & \best $86.18_{(\pm 0.51)}$     \\
                         & \textbf{FM$^2$-Large} & \best $79.56_{(\pm 0.42)}$ & \best $87.18_{(\pm 0.16)}$ & \best $90.19_{(\pm 0.29)}$ & \best $80.28_{(\pm 0.48)}$ &\best  $84.40_{(\pm 0.95)}$ & \best $89.44_{(\pm 0.40)}$ &\best  $74.68_{(\pm 0.36)}$ & \best $84.63_{(\pm 0.23)}$ & \best $87.69_{(\pm 0.30)}$    \\
\bottomrule
\end{tabular}}
\label{tab:main_results}
\end{table*}

\begin{proof}
Write $h(\mX) = \gamma_d f_d(\mX) + \gamma_c f_c(\mX)$ where $f_d = \gM_d \circ \gV$ and $f_c = \gM_c \circ \gV$. For any convex loss $\ell$, Jensen's inequality gives:
\begin{equation}
    \E[\ell(h(\mX), y)] \leq \gamma_d \E[\ell(f_d(\mX), y)] + \gamma_c \E[\ell(f_c(\mX), y)].
\end{equation}
Each expert risk decomposes as $\E[\ell(f_d, y)] = R_d^* + \E[\ell(f_d, y) - \ell(f_d^*, y)]$, where $f_d^*$ is the Bayes-optimal domain expert.
For a bounded loss $\ell \in [0,1]$, the excess risk satisfies $|\E[\ell(f_d, y)] - R_d^*| \leq d_{\mathrm{TV}}(P_{f_d}, P_{f_d^*})$, where $P_f$ denotes the joint law of $(f(\mX), y)$ and $d_{\mathrm{TV}}$ the total variation distance. Since both experts share the Observer $\gV$, their outputs couple through $P_d$ (DMoE routing) and $P_c$ (CMoE routing). Applying Pinsker's inequality to the Jensen--Shannon midpoint $M = (P_d + P_c)/2$ gives $d_{\mathrm{JS}}(P_d\|P_c) \geq \tfrac{1}{2}d_{\mathrm{TV}}(P_d, P_c)^2$, hence $d_{\mathrm{TV}}(P_d, P_c) \leq \sqrt{2\,d_{\mathrm{JS}}(P_d \| P_c)}$. The cross-term in the convex combination carries a factor $2\sqrt{\gamma_d \gamma_c}$ (by Cauchy--Schwarz on the weighted excess risks), giving $2\sqrt{2\gamma_d\gamma_c\,d_{\mathrm{JS}}(P_d\|P_c)}$. Finally, the Observer's finite capacity introduces an approximation error $\epsilon_{\gV} = \inf_{\gV'}\E\|\gV(\mX) - \gV'(\mX)\|^2$, which is architecture-dependent and decreases with model capacity.
\end{proof}
\vspace{-4pt}
Proposition~\ref{prop:generalization} yields three insights: (1)~small $R_d^*,R_c^*$ justify the dual MoE design; (2)~HMA keeps $d_{\mathrm{JS}}$ small by encouraging complementary rather than conflicting specialization; and (3)~the CEL objective (Eq.~\eqref{eq:cel}) directly reduces $\epsilon_{\gV}$ by aligning visual features with caption embeddings, forcing the Observer beyond classification supervision toward finer-grained diagnostic representations.

\section{Experiments}

\subsection{Federated Medical Image Classification}

\paragraph{\bf Datasets.}
We construct the large-scale \textbf{M}edical \textbf{I}maging \textbf{M}odality \textbf{H}eterogeneity (\textbf{MIMH}) benchmark from five public MedMNIST datasets~\cite{yang2023medmnist}: RetinaMNIST (Fundus Camera, 1,600 images), BloodMNIST (Blood Cell Microscope, 17,092 images), TissueMNIST (Kidney Cortex Microscope, 236,386 images), PathMNIST (Colon Pathology, 107,180 images), and DermaMNIST (Dermatoscope, 10,015 images), all at $224 \times 224$ resolution, loaded from the native MedMNIST+ release and \emph{not} the $28 \times 28$ thumbnails commonly associated with MedMNIST. We create four dataset configurations (\textbf{Fig.~\ref{fig:dataset}}): \textbf{MIMH-5/4/3-Domain}: 10 clients with overlapping modalities, using $\texttt{Dirichlet}(\varphi, \rvp)$~\cite{wu2023bold} to control non-IID severity ($\varphi \in \{0.1, 0.3, 100\}$); \textbf{MIMH-SA} (Stand-Alone): 5 clients, each holding data from a single unique modality with zero overlap, representing the hardest setting.
\begin{figure}[tbh]
    \centering
    \includegraphics[width=.9\columnwidth]{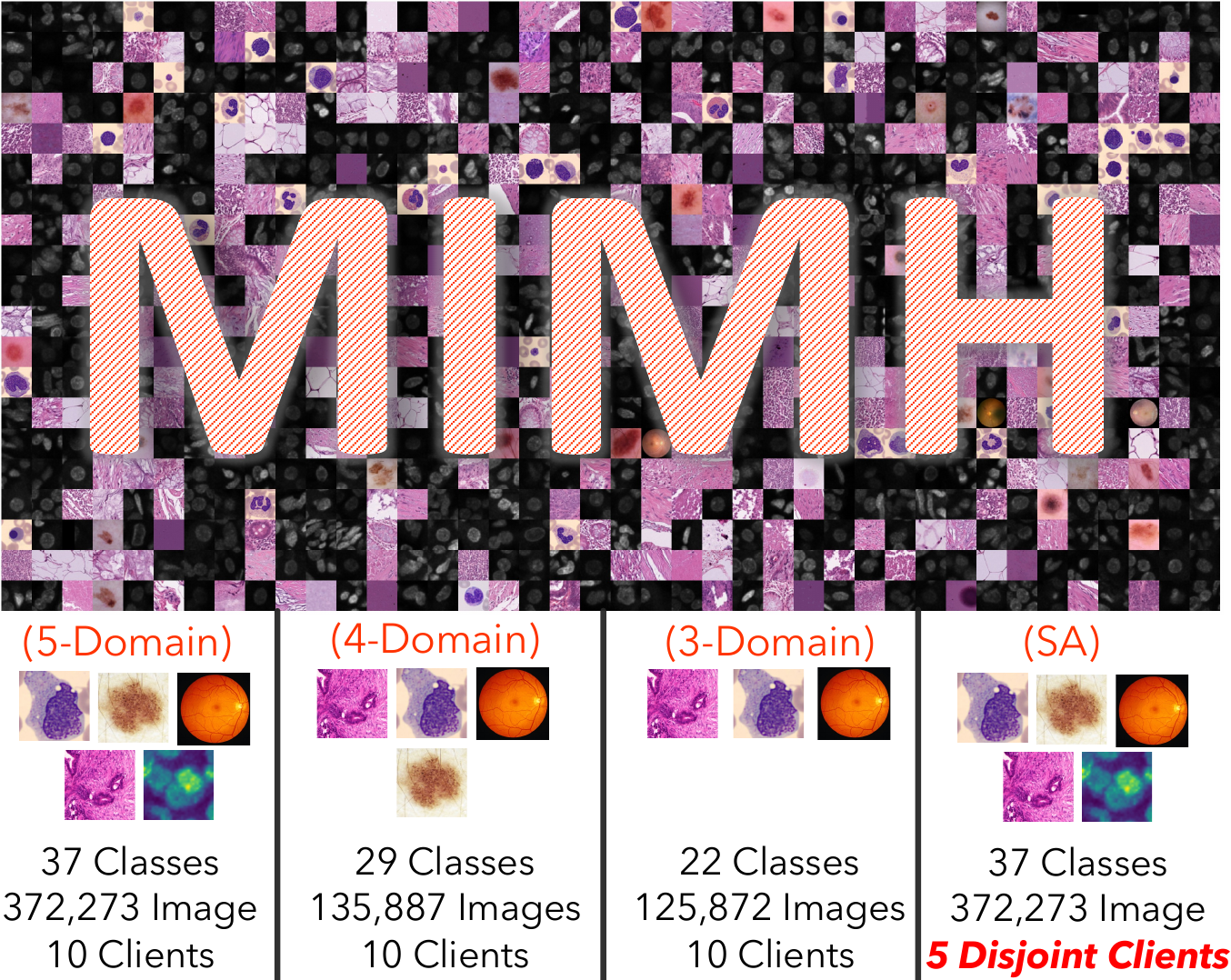}
    \caption{MIMH with four configs: MIMH-5/4/3-Domain (overlapping modalities) and MIMH-SA (non-overlapping).}
    \vspace{-6pt}
    \Description{Illustration of the MIMH dataset construction showing four configurations with different medical imaging modalities assigned to clients.}
    \label{fig:dataset}
\end{figure}

\paragraph{\bf Baselines and Implementation.}
We compare FM$^2$ against 6 advanced FL methods: FedAvg~\cite{mcmahan2017communication}, FedProx~\cite{li2020federated}, PerFedAvg~\cite{fallah2020personalized}, FedBN~\cite{li2021fedbn}, FedProto~\cite{tan2022fedproto}, and FedRep~\cite{collins2021exploiting}. All methods use ResNet-4~\cite{he2016deep} with SGD~\cite{robbins1951stochastic} (lr $= 5 \times 10^{-3}$), 100 communication rounds, and 1 local epoch. FM$^2$'s two-stage protocol (50 Observer + 50 MoE rounds, 1 epoch each) matches all baselines' total budget; see Appendix~A for a full computational breakdown. We set $\gamma_d = \gamma_c = 0.5$ and $\lambda = 0.1$. FM$^2$ is instantiated at three scales: FM$^2$-Tiny/Medium/Large ($2\times$/$4\times$/$8\times$ MLP layers per expert). Results are averaged over five runs with standard deviations reported. All experiments on MIMH datasets are conducted on an NVIDIA A5500.

\paragraph{\bf Main Results.}
Table~\ref{tab:main_results} reports classification accuracy on MIMH-5/4/3-Domain. FM$^2$ consistently outperforms all baselines across most settings. The performance improves as expert capacity increases (Tiny $\to$ Medium $\to$ Large), confirming that scaling the lightweight MoE modules is an effective strategy.

\paragraph{\bf Out-of-Modality Generalization.}
We evaluate two generalization aspects on MIMH-SA (Table~\ref{tab:outofmoda}): \textbf{E1}, can modalities benefit from each other when clients hold completely disjoint modalities? \textbf{E2}, can FM$^2$ generalize to \emph{unseen} modalities? For E2, we introduce three held-out MedMNIST datasets: PneumoniaMNIST~\cite{kermany2018identifying} (Chest X-Ray, 5,856), OCTMNIST~\cite{kermany2018identifying} (Retinal OCT, 109,309), and OrganAMNIST~\cite{xu2019efficient} (Abdominal CT, 58,830), fine-tuning with only 5\% of data. Results show: (1)~baselines collapse on E1 due to inability to capture cross-modality invariances, while FM$^2$ significantly outperforms even local-only training; (2)~on E2, FM$^2$ transfers effectively to all three unseen modalities, demonstrating genuine foundation model capabilities.

\begin{table}[tbh]
\centering
\caption{Classification accuracy on MIMH-SA. Due to limited clients, we fixed the join ratio $r=100\%$. \colorbox{bestbg}{Blue}: the best. `-E' denotes the ensemble among personalized models.}
\resizebox{1\columnwidth}{!}{
\begin{tabular}{cc|cc*{3}{c}}
\toprule
\multirow{2}{*}{\textbf{Method}} & \multirow{2}{*}{\textbf{MIMH-SA}} & \multirow{2}{*}{\textbf{Method}} & \multicolumn{3}{c}{\textbf{MIMH-SA (Cross-Modalities Adaption)}}\\
\cmidrule(lr){4-6} 
 & & & Pneumonia (2-Class) & OCT (4-Class) & Organ (11-Class) \\
\midrule
FedAvg      & $11.40_{(\pm0.35)}$ & FedAvg & $86.13_{(\pm0.22)}$  & $69.94_{(\pm0.14)}$  & $64.94_{(\pm0.27)}$  \\
FedProx     & $11.23_{(\pm0.12)}$  & FedProx  & $87.22_{(\pm0.32)}$ & $73.53_{(\pm0.12)}$ & $68.93_{(\pm0.20)}$ \\
PerFedAvg   & $20.90_{(\pm0.37)}$ & PerFedAvg-E & $87.90_{(\pm0.12)}$ &  $77.98_{(\pm0.40)}$& $71.31_{(\pm0.10)}$\\
FedBN       & $13.00_{(\pm0.29)}$ & FedBN-E & $83.00_{(\pm0.10)}$ & $71.05_{(\pm0.19)}$ &  $68.99_{(\pm0.04)}$   \\
FedProto    & $59.66_{(\pm0.37)}$ & FedProto-E & $90.97_{(\pm0.33)}$ & $80.81_{(\pm0.54)}$ & $78.42_{(\pm0.15)}$   \\
FedRep      & $30.45_{(\pm0.19)}$ & FedRep-E &$89.46_{(\pm0.70)}$  & $74.57_{(\pm0.61)}$ & $74.89_{(\pm0.19)}$  \\ 
\midrule
Local & $69.52_{(\pm0.36)}$ & /& /&/ &/  \\ 
\midrule
\textbf{FM$^2$-Tiny} & \best $73.66_{(\pm0.56)}$& \textbf{FM$^2$-Tiny-E} & \best $96.10_{(\pm0.43)}$ & \best $90.49_{(\pm0.08)}$ &  \best $87.99_{(\pm0.75)}$ \\
\textbf{FM$^2$-Medium} &\best $75.79_{(\pm0.36)}$ &\textbf{FM$^2$-Medium-E} & \best $96.67_{(\pm0.14)}$ & \best $92.01_{(\pm0.36)}$ &  \best $89.85_{(\pm0.64)}$     \\
\textbf{FM$^2$-Large} & \best $76.11_{(\pm0.29)}$&\textbf{FM$^2$-Large-E} & \best $98.06_{(\pm0.30)}$ & \best $93.21_{(\pm0.51)}$ &  \best $91.00_{(\pm0.39)}$     \\ 
\bottomrule
\end{tabular}}
\label{tab:outofmoda}
\vspace{-6pt}
\end{table}

\subsection{Caption-Enhanced Learning}

\paragraph{\bf Dataset and Setup.}
To evaluate the multimodal extension in Sec.~\ref{sec:multimodal}, we augment every image in the MIMH benchmark with a concise diagnostic caption generated by GPT-4o~\cite{openai2024gpt4o} following the ReAct-based pipeline in Appendix~D (e.g., \emph{``A retinal fundus image showing mild drusen deposits consistent with early AMD''}). The visual backbone is upgraded to CLIP ViT-B/16~\cite{radford2021learning}, pretrained on natural image--text pairs, while caption embeddings are extracted by PubMedBERT~\cite{gu2021domain} and fused with image features through a two-head cross-attention layer (dim$=$256). The FM$^2$ dual MoE modules then operate on the fused representation, and an auxiliary InfoNCE contrastive loss aligns image and caption embeddings locally before aggregation. Training follows the protocol in Sec.~4.1, except that we use AdamW~\cite{loshchilov2019decoupled} (lr$=10^{-4}$), 50 communication rounds, and fix $r{=}100\%$. For cross-modality adaptation, we reuse the three held-out datasets from Sec.~4.1 with 5\% fine-tuning data.

\begin{table}[tbh]
\centering
\caption{Classification accuracy with Caption-Enhanced Learning (CEL). \colorbox{bestbg}{\textbf{Blue}}: best, \colorbox{secondbg}{\underline{Peach}}: second-best.}
\vspace{-6pt}
\resizebox{1\columnwidth}{!}{
\begin{tabular}{clcccc}
\toprule
& \textbf{Method} & \multicolumn{2}{c}{\textbf{MIMH-5-Domain}} & \textbf{MIMH-SA} & \textbf{Cross-Mod.} \\
\cmidrule(lr){3-4}
& & $\varphi{=}0.1$ & $\varphi{=}100$ & & \textbf{Avg} \\
\midrule
\multirow{10}{*}{\rotatebox{90}{Image-only}}
& FedAvg~\cite{mcmahan2017communication}       & $62.38_{(\pm0.42)}$ & $70.58_{(\pm0.33)}$ & $38.47_{(\pm0.28)}$ & $68.32$ \\
& FedProx~\cite{li2020federated}                & $64.15_{(\pm0.31)}$ & $71.23_{(\pm0.25)}$ & $39.21_{(\pm0.35)}$ & $69.15$ \\
& PerFedAvg~\cite{fallah2020personalized}       & $67.41_{(\pm0.55)}$ & $74.82_{(\pm0.41)}$ & $44.68_{(\pm0.52)}$ & $73.47$ \\
& FedBN~\cite{li2021fedbn}                      & $65.82_{(\pm0.37)}$ & $72.48_{(\pm0.38)}$ & $41.53_{(\pm0.30)}$ & $70.41$ \\
& FedProto~\cite{tan2022fedproto}               & $80.92_{(\pm0.58)}$ & \secondbest $90.45_{(\pm0.44)}$ & \secondbest $72.63_{(\pm0.41)}$ & \secondbest $82.07$ \\
& FedRep~\cite{collins2021exploiting}            & \secondbest $81.74_{(\pm0.46)}$ & $84.82_{(\pm0.52)}$ & $54.85_{(\pm0.33)}$ & $76.93$ \\
& FedPer~\cite{arivazhagan2019federated}        & $78.53_{(\pm0.40)}$ & $82.37_{(\pm0.36)}$ & $51.27_{(\pm0.44)}$ & $74.82$ \\
& FM$^2$-Tiny                                   & \best $85.47_{(\pm0.43)}$ & \best $91.85_{(\pm0.30)}$ & \best $80.31_{(\pm0.50)}$ & \best $87.62$ \\
& FM$^2$-Medium                                 & \best $87.30_{(\pm0.39)}$ & \best $93.17_{(\pm0.28)}$ & \best $82.14_{(\pm0.47)}$ & \best $89.25$ \\
& FM$^2$-Large                                  & \best $88.62_{(\pm0.35)}$ & \best $94.21_{(\pm0.24)}$ & \best $83.59_{(\pm0.42)}$ & \best $90.48$ \\
\midrule
\multirow{10}{*}{\rotatebox{90}{+\,CEL}}
& FedAvg         & $66.85_{(\pm0.50)}$ & $73.35_{(\pm0.38)}$ & $45.82_{(\pm0.44)}$ & $74.58$ \\
& FedProx        & $68.42_{(\pm0.44)}$ & $74.08_{(\pm0.32)}$ & $46.95_{(\pm0.40)}$ & $75.33$ \\
& PerFedAvg      & $71.58_{(\pm0.52)}$ & $77.45_{(\pm0.39)}$ & $51.93_{(\pm0.48)}$ & $78.84$ \\
& FedBN          & $70.17_{(\pm0.46)}$ & $75.21_{(\pm0.41)}$ & $48.65_{(\pm0.43)}$ & $76.25$ \\
& FedProto       & $84.17_{(\pm0.53)}$ & \secondbest $92.10_{(\pm0.37)}$ & \secondbest $78.50_{(\pm0.55)}$ & \secondbest $86.43$ \\
& FedRep         & \secondbest $85.03_{(\pm0.48)}$ & $87.25_{(\pm0.45)}$ & $61.42_{(\pm0.49)}$ & $81.07$ \\
& FedPer         & $82.36_{(\pm0.45)}$ & $85.14_{(\pm0.40)}$ & $57.83_{(\pm0.47)}$ & $80.15$ \\
& \textbf{FM$^2$-Tiny}    & \best $89.16_{(\pm0.45)}$ & \best $93.71_{(\pm0.27)}$ & \best $85.20_{(\pm0.51)}$ & \best $92.15$ \\
& \textbf{FM$^2$-Medium}  & \best $91.43_{(\pm0.41)}$ & \best $95.08_{(\pm0.22)}$ & \best $87.25_{(\pm0.48)}$ & \best $93.81$ \\
& \textbf{FM$^2$-Large}   & \best $92.58_{(\pm0.33)}$  & \best $95.92_{(\pm0.19)}$  & \best $88.73_{(\pm0.40)}$  & \best $95.04$ \\
\bottomrule
\end{tabular}}
\label{tab:cel}
\end{table}

\paragraph{\bf Results.}
Table~\ref{tab:cel} compares ten image-only FL baselines with seven caption-enhanced variants. Several observations stand out. First, appending captions yields consistent improvements for every FL method: even the weakest baseline FedAvg gains +4.47\% on the Hard setting and +7.35\% on MIMH-SA, confirming that language grounding provides complementary information regardless of the aggregation strategy. Second, the improvement is most pronounced on MIMH-SA and cross-modality transfer, where visual domains are completely disjoint and captions serve as a shared semantic bridge; FM$^2$-Large+CEL reaches 95.04\% cross-modality average, a +4.56\% margin over its image-only counterpart. Third, the Tiny$\to$Medium$\to$Large scaling pattern is preserved under CEL (85.20$\to$87.25$\to$88.73 on MIMH-SA), indicating that MoE capacity scaling and caption supervision contribute orthogonally to performance. Among all image-only methods, FM$^2$-Large already surpasses the best non-FM$^2$ baseline (FedProto) by a large margin, and the addition of CEL further widens this gap.

\subsection{Visual Question Answering}

\paragraph{\bf Datasets and Setup.}
Following the federated medical VQA setup of OmniFM~\cite{liu2026omnifm}, we construct three settings using SLAKE~\cite{liu2021slake}, VQA-RAD~\cite{lau2018dataset}, and VQA-Med 2019/2020/2021~\cite{ben2019vqa,ben2021overview}. \textbf{Task~1} uses SLAKE with modality-based IID/non-IID splits across three clients under varying LLaVA fine-tuning strategies. \textbf{Task~2} constructs eight modality-specific clients to test cross-domain DMoE specialization. \textbf{Task~3} treats each dataset as an independent mixed-modality client (five total) under different PEFT regimes. All tasks use LLaVA-1.5 with CLIP ViT-B/32 augmented by our dual MoE modules and LLaMA-3.2-3B fine-tuned via LoRA~\cite{hu2021lora} ($r{=}8$, $\alpha{=}32$). FedDAT, PromptFL, and F$^3$OCUS are re-implemented on the identical LLaVA-1.5 backbone with the same LoRA configuration, with only their respective federated adaptation strategies applied, ensuring backbone-level parity across all methods.

\begin{table}[tbh]
    \centering
    \caption{VQA accuracy on Task~1 (SLAKE) under IID/non-IID settings. F-C/F-L: connector/LLM-only tuning; F-CL/F-2stage: joint/sequential tuning. \colorbox{bestbg}{\textbf{Blue}}: best, \colorbox{secondbg}{\underline{Peach}}: second-best.}
    \vspace{-6pt}
    \setlength{\tabcolsep}{2pt}
    \setlength{\extrarowheight}{2pt}
    \resizebox{\columnwidth}{!}{
    \begin{tabular}{l*{4}{cc}}
        \toprule
        \multirow{2}{*}{\makecell[l]{\textbf{Method}}} &
        \multicolumn{2}{c}{\textbf{F-C}} &
        \multicolumn{2}{c}{\textbf{F-L}} &
        \multicolumn{2}{c}{\textbf{F-CL}} &
        \multicolumn{2}{c}{\textbf{F-2stage}} \\
        \cmidrule(lr){2-3} \cmidrule(lr){4-5} \cmidrule(lr){6-7} \cmidrule(lr){8-9}
        & IID & Non-IID & IID & Non-IID & IID & Non-IID & IID & Non-IID \\
        \midrule
        FedAvg~\cite{mcmahan2017communication} & 0.783 & 0.775 & 0.806 & 0.802 & 0.823 & 0.827 & 0.811 & 0.814 \\
        FedProx~\cite{li2020federated} & 0.734 & 0.750 & 0.800 & 0.780 & 0.816 & 0.796 & 0.773 & 0.785 \\
        FedAdam~\cite{reddi2020adaptive} & 0.741 & 0.735 & 0.783 & 0.771 & 0.777 & 0.774 & 0.782 & 0.777 \\
        FedAvgM~\cite{hsu2019measuring} & 0.754 & 0.747 & 0.789 & 0.786 & 0.784 & 0.768 & 0.793 & 0.794 \\
        FedYogi~\cite{reddi2020adaptive} & 0.745 & 0.736 & 0.782 & 0.769 & 0.783 & 0.774 & 0.785 & 0.782 \\
        FedDyn~\cite{acar2021federated} & 0.795 & 0.783 & \secondres{0.812} & 0.800 & 0.820 & 0.827 & 0.809 & 0.816 \\
        FedPer~\cite{arivazhagan2019federated} & \secondres{0.799} & \secondres{0.790} & 0.809 & \secondres{0.814} & \secondres{0.827} & \secondres{0.833} & \secondres{0.818} & \secondres{0.824} \\
        \midrule
        \textbf{FM$^2$ (Ours)} & \boldres{0.811} & \boldres{0.805} & \boldres{0.818} & \boldres{0.828} & \boldres{0.838} & \boldres{0.844} & \boldres{0.829} & \boldres{0.837} \\
        \bottomrule
    \end{tabular}}
    \label{tab:task1_vqa}
\end{table}

\begin{table}[tbh]
    \centering
    \caption{VQA accuracy on Task~2 (complete modality heterogeneity, 8 clients). C1: CT, C2: Ultrasound, C3: OCT, C4: Fundus, C5: Microscopy, C6: Histopath., C7: Dermatoscopy, C8: X-ray. \colorbox{bestbg}{\textbf{Blue}}: best, \colorbox{secondbg}{\underline{Peach}}: second-best.}
    \vspace{-6pt}
    \setlength{\tabcolsep}{2pt}
    \setlength{\extrarowheight}{3pt}
    \resizebox{\columnwidth}{!}{
    \begin{tabular}{lccccccccc}
        \toprule
        \textbf{Method} & \multicolumn{8}{c}{\textbf{Task-Specific Performance}} & \textbf{Avg.} \\
        \cmidrule(lr){2-9}
        & C1 & C2 & C3 & C4 & C5 & C6 & C7 & C8 & \\
        \midrule
        FedAvg~\cite{mcmahan2017communication} & 88.67 & 72.59 & 70.86 & 85.93 & 71.09 & 89.45 & 66.62 & 76.48 & 77.71 \\
        FedProx~\cite{li2020federated} & 88.92 & 73.05 & 71.03 & 86.10 & 71.40 & 89.20 & 66.81 & 76.32 & 77.85 \\
        FedAdam~\cite{reddi2020adaptive} & 89.10 & \secondres{74.22} & 71.58 & 85.70 & 71.85 & 89.80 & 67.20 & 76.90 & 78.17 \\
        FedDyn~\cite{acar2021federated} & 88.75 & 73.90 & 71.92 & \secondres{86.35} & 71.55 & \secondres{90.02} & 67.45 & 76.70 & 78.08 \\
        FedPer~\cite{arivazhagan2019federated} & \secondres{89.43} & 74.10 & \secondres{72.21} & 85.95 & \secondres{72.30} & 89.90 & \secondres{67.82} & \secondres{77.52} & \secondres{78.40} \\
        \midrule
        \textbf{FM$^2$ (Ours)} & \boldres{90.28} & \boldres{75.71} & \boldres{73.05} & \boldres{87.42} & \boldres{73.35} & \boldres{90.48} & \boldres{68.43} & \boldres{78.18} & \boldres{79.60} \\
        \bottomrule
    \end{tabular}}
    \label{tab:vqa_task2}
    \vspace{-6pt}
\end{table}

\begin{table}[tbh]
    \centering
    \caption{VQA accuracy on Task~3 (mixed-modality heterogeneity, 5 dataset-level clients). VM2019/2020/2021 is VQA-Med. VR: VQA-RAD. \colorbox{bestbg}{\textbf{Blue}}: best, \colorbox{secondbg}{\underline{Peach}}: second-best.}
    \vspace{-6pt}
    \resizebox{1\columnwidth}{!}{
    \begin{tabular}{lcccccc}
        \toprule
        \textbf{Method} & \textbf{SLAKE} & \textbf{VM2019} & \textbf{VM2020} & \textbf{VM2021} & \textbf{VR} & \textbf{Average} \\
        \midrule
        \textbf{Full Tuning} \\
        \qquad $\triangleright$ \color{midgray}{FedAvg} & 77.45 & 67.25 & 15.06 & 22.00 & 41.52 & 44.66 \\
        \qquad $\triangleright$ \color{midgray}{FedProx} & 77.47 & 67.00 & 14.87 & 22.12 & 41.00 & 44.33 \\
        \qquad $\triangleright$ \color{midgray}{FedAdam} & 77.00 & 67.89 & 15.12 & 22.44 & 42.42 & 44.77 \\
        \qquad $\triangleright$ \color{midgray}{FedDyn} & 78.24 & 69.00 & 16.21 & 21.65 & 42.02 & 45.02 \\
        \qquad $\triangleright$ \color{midgray}{FedPer} & \secondres{79.23} & \secondres{69.20} & \secondres{16.77} & \secondres{22.54} & \secondres{43.43} & \secondres{46.03} \\
        \qquad \textcolor{blue!70!black}{$\blacktriangleright$} \color{midgray}{\textbf{FM$^2$ (Ours)}} & \boldres{82.58} & \boldres{70.42} & \boldres{18.45} & \boldres{22.65} & \boldres{43.62} & \boldres{47.54} \\
        \textbf{Adapter}~\cite{houlsby2019parameter} \\
        \qquad $\triangleright$ \color{midgray}{FedAvg} & 72.82 & 64.45 & 11.56 & 21.00 & 38.35 & 41.64 \\
        \qquad $\triangleright$ \color{midgray}{FedProx} & 73.10 & 64.70 & 11.72 & 21.18 & 38.60 & 41.86 \\
        \qquad $\triangleright$ \color{midgray}{FedAdam} & 73.05 & 65.12 & 12.08 & 21.47 & 39.12 & 42.17 \\
        \qquad $\triangleright$ \color{midgray}{FedDyn} & 74.02 & 65.88 & 12.41 & 21.39 & 39.05 & 42.55 \\
        \qquad $\triangleright$ \color{midgray}{FedPer} & \secondres{75.10} & \secondres{66.30} & \secondres{12.95} & \secondres{21.76} & \secondres{39.70} & \secondres{43.16} \\
        \qquad \textcolor{blue!70!black}{$\blacktriangleright$} \color{midgray}{\textbf{FM$^2$ (Ours)}} & \boldres{78.15} & \boldres{67.72} & \boldres{14.08} & \boldres{22.12} & \boldres{40.42} & \boldres{44.50} \\
        \textbf{LayerNorm}~\cite{basu2024strong} \\
        \qquad $\triangleright$ \color{midgray}{FedAvg} & 69.53 & 62.22 & 10.59 & 22.00 & 36.89 & 40.24 \\
        \qquad $\triangleright$ \color{midgray}{FedProx} & 69.80 & 62.40 & 10.75 & 22.12 & 37.05 & 40.42 \\
        \qquad $\triangleright$ \color{midgray}{FedAdam} & 70.02 & 62.95 & 11.08 & 22.35 & 37.48 & 40.78 \\
        \qquad $\triangleright$ \color{midgray}{FedDyn} & 70.55 & 63.40 & 11.32 & 22.28 & 37.60 & 41.03 \\
        \qquad $\triangleright$ \color{midgray}{FedPer} & \secondres{71.22} & \secondres{63.88} & \secondres{11.70} & \secondres{22.55} & \secondres{38.10} & \secondres{41.49} \\
        \qquad \textcolor{blue!70!black}{$\blacktriangleright$} \color{midgray}{\textbf{FM$^2$ (Ours)}} & \boldres{73.42} & \boldres{64.88} & \boldres{12.55} & \boldres{22.85} & \boldres{38.82} & \boldres{42.50} \\
        \textbf{LoRA}~\cite{hu2021lora} \\
        \qquad $\triangleright$ \color{midgray}{FedAvg} & 57.73 & 60.18 & 4.59 & 15.00 & 31.16 & 33.73 \\
        \qquad $\triangleright$ \color{midgray}{FedProx} & 58.10 & 60.05 & 4.80 & 15.21 & 31.45 & 33.92 \\
        \qquad $\triangleright$ \color{midgray}{FedAdam} & 58.02 & 60.72 & 5.10 & 15.48 & 31.90 & 34.24 \\
        \qquad $\triangleright$ \color{midgray}{FedDyn} & 58.30 & 60.70 & 5.05 & 15.10 & 31.60 & 34.15 \\
        \qquad $\triangleright$ \color{midgray}{FedPer} & \secondres{60.20} & \secondres{61.55} & \secondres{5.80} & \secondres{15.70} & \secondres{32.40} & \secondres{35.13} \\
        \qquad \textcolor{blue!70!black}{$\blacktriangleright$} \color{midgray}{\textbf{FM$^2$ (Ours)}} & \boldres{62.72} & \boldres{62.35} & \boldres{6.65} & \boldres{15.88} & \boldres{32.32} & \boldres{35.98} \\
        \textbf{Bias}~\cite{cai2020tinytl} \\
        \qquad $\triangleright$ \color{midgray}{FedAvg} & 68.25 & 55.61 & 11.17 & 17.00 & 35.47 & 37.50 \\
        \qquad $\triangleright$ \color{midgray}{FedProx} & 68.60 & 55.40 & 11.35 & 17.12 & 35.60 & 37.61 \\
        \qquad $\triangleright$ \color{midgray}{FedAdam} & 68.10 & 56.05 & 11.70 & \secondres{17.35} & \secondres{36.10} & 37.86 \\
        \qquad $\triangleright$ \color{midgray}{FedDyn} & 68.70 & 56.00 & 11.45 & 17.05 & 35.70 & 37.78 \\
        \qquad $\triangleright$ \color{midgray}{FedPer} & \secondres{70.10} & \secondres{56.88} & \secondres{12.00} & 17.00 & 36.00 & \secondres{38.40} \\
        \qquad \textcolor{blue!70!black}{$\blacktriangleright$} \color{midgray}{\textbf{FM$^2$ (Ours)}} & \boldres{71.52} & \boldres{57.42} & \boldres{12.55} & \boldres{17.85} & \boldres{36.48} & \boldres{39.16} \\
        \midrule
        FedDAT~\cite{chen2024feddat} & 72.79 & 63.58 & 12.46 & 23.00 & 38.91 & 42.15 \\
        PromptFL~\cite{guo2023promptfl} & 65.63 & 57.56 & 4.78 & 15.00 & 40.26 & 36.65 \\
        F$^3$OCUS~\cite{saha2025f} & \secondres{74.69} & \secondres{63.82} & \secondres{12.84} & \secondres{24.00} & \secondres{42.30} & \secondres{43.53} \\
        \textbf{FM$^2$ (Ours)} & \boldres{79.45} & \boldres{68.52} & \boldres{15.33} & \boldres{23.75} & \boldres{42.65} & \boldres{45.94} \\
        \bottomrule
    \end{tabular}}
    \label{tab:vqa_task3}
    \vspace{-6pt}
\end{table}

\begin{figure*}[tbh]
  \centering
  \includegraphics[width=.925\textwidth]{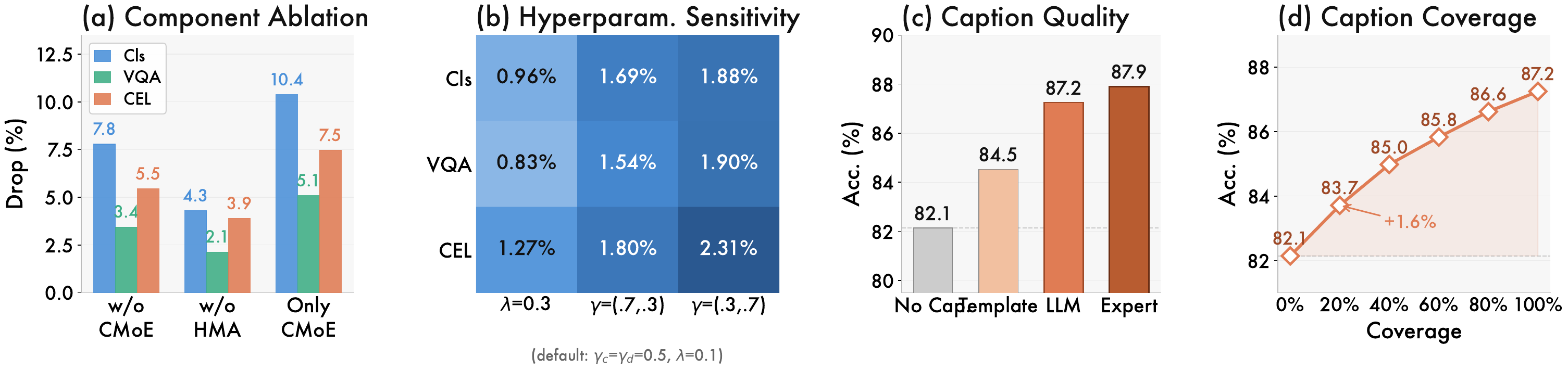}
  \caption{Ablation and analysis across all three tasks. \textbf{(a)}~Component ablation. \textbf{(b)}~Hyperparameter sensitivity across different $(\gamma_c, \gamma_d, \lambda)$ configurations. \textbf{(c)}~Impact of caption quality on CEL. \textbf{(d)}~Effect of partial caption availability on CEL.}
  \label{fig:ablation}
\end{figure*}
\paragraph{\bf Results.}
FM$^2$ consistently outperforms all baselines across all three tasks and tuning strategies (Tables~\ref{tab:task1_vqa}--\ref{tab:vqa_task3}). On Task~1, gains are largest under F-CL (0.844 vs.\ 0.833 for the best baseline in Non-IID), validating that HMA effectively mitigates distribution discrepancies during VQA training. On Task~2, FM$^2$ achieves the highest average accuracy of 79.60, confirming DMoE's cross-domain specialization. On Task~3, FM$^2$ maintains consistent advantages across all PEFT regimes and outperforms specialized federated VLM methods FedDAT~\cite{chen2024feddat}, PromptFL~\cite{guo2023promptfl}, and F$^3$OCUS~\cite{saha2025f} by clear margins.

\subsection{Framework Analysis}

We analyze FM$^2$ from six complementary angles: \textbf{(i)}~component ablation (Fig.~\ref{fig:ablation}a), \textbf{(ii)}~hyperparameter sensitivity (Fig.~\ref{fig:ablation}b), \textbf{(iii)}~caption quality and coverage (Fig.~\ref{fig:ablation}c--d), \textbf{(iv)}~communication cost (Table~\ref{tab:comm}), \textbf{(v)}~convergence behavior, and \textbf{(vi)}~learned feature structure (Fig.~\ref{fig:analysis}). Unless noted, ablation and sensitivity results span three tasks: classification (MIMH-5-Domain, $\varphi{=}0.1$, FM$^2$-Tiny), VQA (Task~1, F-CL, Non-IID), and CEL (MIMH-SA, FM$^2$+CEL), to verify that findings generalize across objectives.

\paragraph{\bf Component Ablation.}
We progressively disable FM$^2$'s core components (Fig.~\ref{fig:ablation}a). Among single-component removals, disabling CMoE costs 3.4--7.8\% and disabling HMA costs 2.1--4.3\%. The most harmful variant, \emph{Only CMoE}, removes \emph{both} DMoE and HMA and drops 5.1--10.4\%, showing that personalization alone is insufficient without global domain consensus. Since w/o-HMA and Only-CMoE differ only in DMoE, removing DMoE alone still costs 5.4 and 1.8 points on MIMH-5-Domain and MIMH-SA, so DMoE contributes beyond a personalized CMoE plus the federated backbone. The importance ordering (Only CMoE $>$ w/o CMoE $>$ w/o HMA) is consistent across all three tasks, confirming architectural robustness independent of downstream objective.

\paragraph{\bf Hyperparameter Sensitivity.}
Fig.~\ref{fig:ablation}(b) examines the effect of varying $\gamma_c$, $\gamma_d$, and $\lambda$ away from the default setting ($\gamma_c{=}0.5$, $\gamma_d{=}0.5$, $\lambda{=}0.1$). Across all three tasks, no non-default configuration degrades performance by more than 2.31\%, indicating that FM$^2$ is not sensitive to precise hyperparameter tuning and that the default configuration generalizes reliably across diverse settings and tasks.

\begin{table}[tbh]
\centering
\caption{Communication and computation cost per round (MIMH-5-Domain, $\varphi{=}0.1$, $r{=}100\%$, 10 clients).}
\vspace{-8pt}
\resizebox{\columnwidth}{!}{
\begin{tabular}{lcccc}
\toprule
\textbf{Method} & \textbf{Comm. Params (K)} & \textbf{Comm./Rd (MB)} & \textbf{Time/Rd (s)} & \textbf{Acc.} \\
\midrule
FedAvg~\cite{mcmahan2017communication}   & 520  & 41.6 & 8.2  & 23.05 \\
FedProx~\cite{li2020federated}           & 520  & 41.6 & 9.5  & 38.70 \\
FedProto~\cite{tan2022fedproto}          & 1.3  & 0.1  & 7.8  & 70.84 \\
FedRep~\cite{collins2021exploiting}      & 450  & 36.0 & 9.1  & 74.72 \\
\midrule
FM$^2$-Tiny    & 718  & 57.4  & 12.5 & \best 83.79 \\
FM$^2$-Medium  & 824  & 65.9  & 14.8 & \best 85.66 \\
FM$^2$-Large   & 1035 & 82.8  & 18.3 & \best 86.68 \\
\bottomrule
\end{tabular}}
\label{tab:comm}
\vspace{-8pt}
\end{table}

\paragraph{\bf Caption Quality and Coverage.}
Fig.~\ref{fig:ablation}(c--d) examines CEL robustness on MIMH-SA. LLM-generated captions (87.25\%) nearly match expert annotations (87.89\%) while far exceeding templates (84.53\%), confirming API generation as a cost-effective substitute for manual labeling. Even at 20\% coverage, CEL yields +1.57\% over the image-only baseline, with monotonically increasing but diminishing returns, showing that CEL is beneficial even under the uneven caption availability typical of real federated deployments.
\begin{figure}[tbh]
  \centering
  \includegraphics[width=.95\columnwidth]{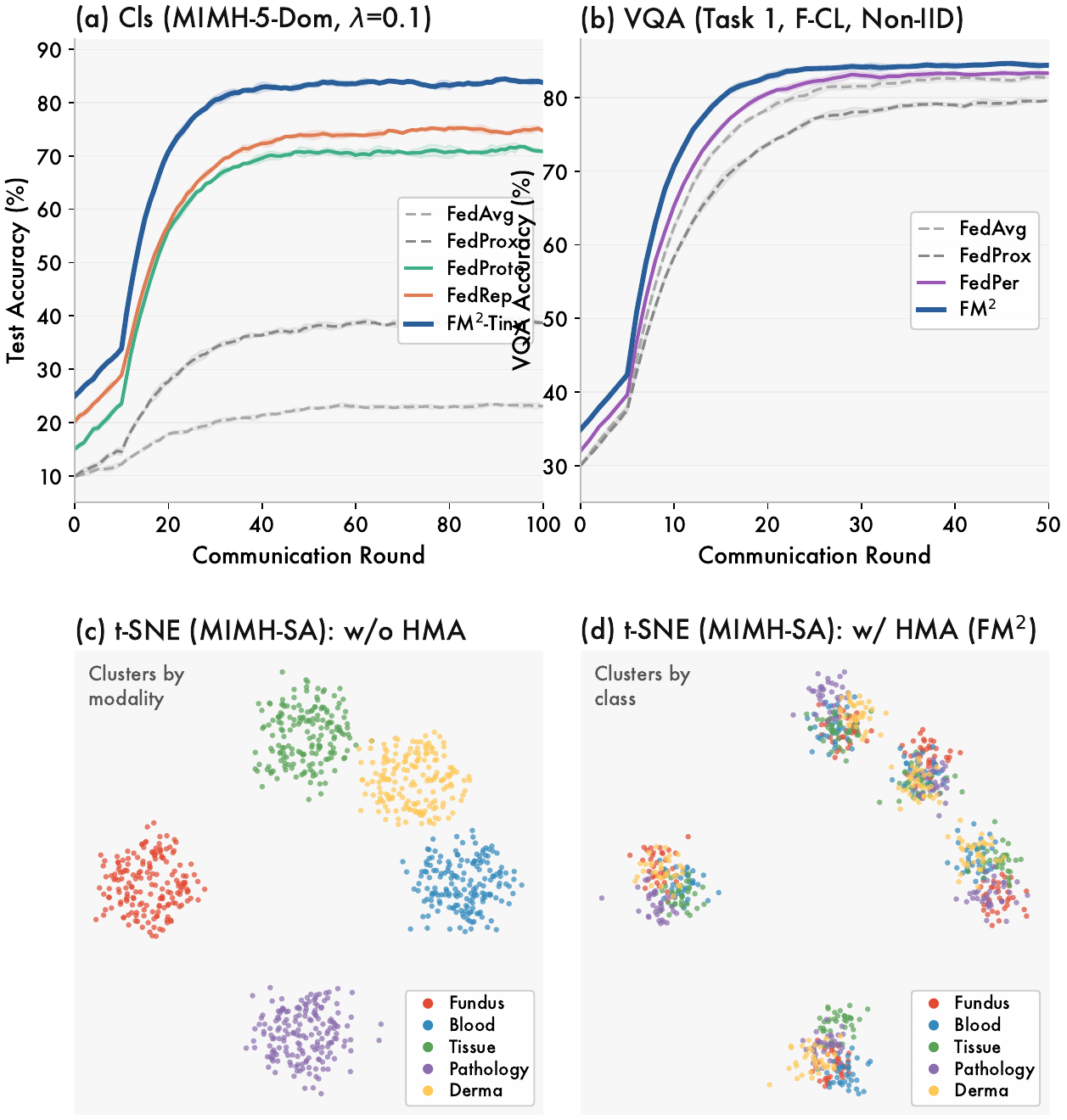}
  \caption{Convergence and feature analysis. \textbf{(a)}~Classification convergence on MIMH-5-Domain ($\varphi{=}0.1$, $r{=}100\%$); shaded bands show $\pm$1 std over 5 runs. \textbf{(b)}~VQA convergence on Task~1 (F-CL, Non-IID, 50 rounds). \textbf{(c)}~t-SNE on MIMH-SA without HMA: features cluster by modality. \textbf{(d)}~With HMA (FM$^2$): features cluster by semantic class across modalities.}
  \label{fig:analysis}
  \vspace{-8pt}
\end{figure}
\paragraph{\bf Communication Cost.}
Table~\ref{tab:comm} reports per-round communication and wall-clock overhead. Because the CMoE module is entirely local (never communicated), the additional communication of FM$^2$-Tiny over FedAvg is 38\% in parameter count (718K vs.\ 520K), a modest price for a 60.75 percentage-point accuracy gain. FedProto achieves the lowest communication (prototype-only exchange) but plateaus at 70.84\%. FM$^2$-Large uses $2\times$ FedAvg's bandwidth yet reaches 86.68\%, representing a fundamentally different accuracy-communication trade-off.

\paragraph{\bf Convergence and Feature Analysis.}
Figure~\ref{fig:analysis} reports convergence and feature structure. On classification (Fig.~\ref{fig:analysis}a), FM$^2$-Tiny reaches 80\% by round~40, a level no baseline achieves at round~100; FedAvg stagnates below 25\% while FedProto and FedRep plateau near 70--75\%. On VQA (Fig.~\ref{fig:analysis}b), FM$^2$ converges fastest to 84.4\%, above FedPer (83.3\%) with tighter variance. t-SNE on MIMH-SA (Fig.~\ref{fig:analysis}c--d) shows that without HMA features cluster by modality; with HMA they reorganize by semantic class, with all five modalities interleaving within each cluster, validating HMA's design objective and the out-of-modality generalization in Tables~\ref{tab:main_results} and~\ref{tab:cel}.

\section{Conclusion}

In this paper, we proposed FM$^2$, a federated multimodal foundation model framework that addresses the under-explored challenge of Imaging Modality Heterogeneity in medical imaging. Through dual Mixture-of-Experts modules and Heterogeneous Modality Alignment with provable convergence guarantees, FM$^2$ disentangles personalized class knowledge from shared domain representations. Extensions to Caption-Enhanced Learning and Federated Medical VQA demonstrate FM$^2$'s versatility as a multimodal foundation model. Extensive experiments on the MIMH benchmark (for classification and CEL) and real-world medical VQA datasets confirm consistent superiority over state-of-the-art baselines, fast convergence, strong out-of-modality generalization, and practical robustness to caption availability and hyperparameter choices.

\begin{acks}
This work was supported by the Guangdong Basic and Applied Basic Research Foundation (No.~2024A1515510031). This work is also supported by the Intelligent Computing Center of Shenzhen University.
\end{acks}

\bibliographystyle{ACM-Reference-Format}
\bibliography{mybib}

\end{document}